\documentclass[twoside,11pt]{article}

\usepackage{jmlr2e}
\usepackage[utf8]{inputenc}
\usepackage[T1]{fontenc}
\usepackage{microtype}
\usepackage{graphicx,color}
\usepackage{latexsym,amsmath,amssymb,amsfonts}
\usepackage{natbib}
\usepackage{theorem}
\usepackage{comment}
\usepackage{algorithm}
\usepackage[noend]{algpseudocode}
\usepackage{multirow}
\usepackage{booktabs}
\usepackage{subcaption}
\usepackage[breaklinks=true]{hyperref}
\usepackage{breakcites}

\hypersetup{colorlinks,
            linkcolor=blue,
            citecolor=blue,
            urlcolor=magenta,
            linktocpage,
            plainpages=false}
            
\algblockdefx{FORALLP}{ENDFAP}[1]%
  {\textbf{for all }#1 \textbf{do in parallel}}%
  {\textbf{end for}}
            
\newtheorem{thm}{Theorem}[section]
\newtheorem{lem}{Lemma}[section]

\newtheorem{cond}{Condition}[section]

\firstpageno{1}

\begin{document}

\title{Nearly Optimal Variational Inference for High Dimensional Regression with Shrinkage Priors}

% \author{\name Marina Meil\u{a} \email mmp@stat.washington.edu \\
%       \addr Department of Statistics\\
%       University of Washington\\
%       Seattle, WA 98195-4322, USA
%       \AND
%       \name Michael I.\ Jordan \email jordan@cs.berkeley.edu \\
%       \addr Division of Computer Science and Department of Statistics\\
%       University of California\\
%       Berkeley, CA 94720-1776, USA}
\author{\name Jincheng Bai \email bai45@purdue.edu \\
      \name Qifan Song \email qfsong@purdue.edu \\
      \name Guang Cheng \email chengg@purdue.edu \\
      \addr  Department of Statistics \\
      Purdue University\\
      West Lafayette, IN 47906, USA}

\maketitle

\begin{abstract}
%   The Abstract paragraph should be indented 0.25 inch (1.5 picas) on
%   both left and right-hand margins. Use 10~point type, with a vertical
%   spacing of 11~points. The \textbf{Abstract} heading must be centered,
%   bold, and in point size 12. Two line spaces precede the
%   Abstract. The Abstract must be limited to one paragraph.
We propose a variational Bayesian (VB) procedure for high-dimensional linear model inferences with heavy tail shrinkage priors, such as student-$t$ prior. Theoretically, we establish the consistency of the proposed VB method and prove that under the proper choice of prior specifications, the contraction rate of the VB posterior is nearly optimal. It justifies the validity of VB inference as an alternative of Markov Chain Monte Carlo (MCMC) sampling. Meanwhile, comparing to conventional MCMC methods, the VB procedure achieves much higher computational efficiency, which greatly alleviates the computing burden for modern machine learning applications such as massive data analysis. Through numerical studies, we demonstrate that the proposed VB method leads to shorter computing time, higher estimation accuracy, and lower variable selection error than competitive sparse Bayesian methods.
\end{abstract}

\section{INTRODUCTION}
High dimensional sparse linear regression is one of the most commonly encountered problems in machine learning and statistics communities \citep{Hastie2001Elements}. In the Bayesian paradigm, this problem is approached by placing sparsity-inducing priors on the regression coefficients. There are mainly two types of priors: the spike-and-slab prior \citep{Mitchell1988Bayesian, George1993variable, Ishwaran2005Spike} and the shrinkage prior \citep{Hans2009Bayesian, Carvalho2010horseshoe, Griffin2012Structuring}. The spike-and-slab prior has been considered as the gold standard for high dimensional linear regression, whose theoretical properties have been thoroughly studied \citep{Johnson2012Bayesian, Song2014Split, Yang2016On, gao2020general}. Although theoretically sound, the posterior sampling cost under spike-and-slab priors could be highly expensive, as it usually requires a tran-dimensional MCMC sampler such as reversible-jump MCMC. Alternatively, shrinkage priors could lead to equally good theoretical properties \citep{Ghosal1999Asymptotic, Armagan2013Posterior, song2017nearly} while enjoying computational efficiency via the use of conjugate Gibbs sampler. 

Although switching to shrinkage prior could reduce the computational burden to some extent, the nature of Bayesian computing (i.e., Markov chain Monte Carlo simulation) inevitably requires a huge number of iterations in order to achieve good mixing behavior and obtain accurate large-sample average. Consequently, people has sought to find frequentist shortcuts for Bayesian estimators. For example, \cite{Rockova2014EMVS} proposed EM algorithm to find posterior modes under the spike-and-slab prior; \cite{Rockova2018spike} obtained the posterior modes by using penalized likelihood estimation; \cite{Bhadra2019horseshooe} searched posterior modes under horseshoe prior via optimization methods. Those approaches are computational-friendly, however completely ignore the distribution information of posterior and can not derive any Bayesian inferences beyond point estimation.

Another computationally convenient alternative to MCMC is the variational inference (VI or VB) \citep{Jordan1999Variational, Blei2017variational}. 
VI can provide an approximate posterior via frequentist optimization, thus it delivers (approximate) distributional inferences within a fairly small number of iterations.
In the context of high dimensional linear regression, \cite{Carbonetto2012Scalable}, \cite{Huang2016variational} and \cite{Ray2020Variational} have proposed algorithms to carry out variational inferences under spike-and-slab priors. Besides their empirical successes, the theoretical properties were also justified. Specifically, \cite{Huang2016variational} showed their algorithm could achieve asymptotic consistency, and \cite{Ray2020Variational} established the oracle inequalities for their VB approximation. Therefore, by employing the scalable variational inference, we would obtain the same theoretical guarantees as using MCMC while hugely reducing the computational cost. 

In this work, we focus on the variational inference for Bayesian regression with shrinkage priors, which further improves the computational efficiency comparing to the one based on the spike-and-slab prior. Meanwhile, by showing the nearly optimal contraction rate of the proposed variational posterior, the validity of the proposed method is justified.

\section{PRELIMINARIES}
\subsection{High-dimensional Regression}

Consider the linear regression model
\begin{equation} \label{eq:model}
\boldsymbol{Y} = \boldsymbol{X}\boldsymbol{\beta} + \sigma\boldsymbol{\epsilon},
\end{equation}
where $\boldsymbol{Y} \in \mathbb{R}^n$ is the response vector, $\boldsymbol{X}=(X_{ij})$ is a $n \times p_n$ design matrix, $\boldsymbol{\beta} =(\beta_1,\dots,\beta_{p_n})'\in \mathbb{R}^{p_n}$ is the coefficient vector and $\boldsymbol{\epsilon} \sim \mathcal{N}(0, I_n)$ is the Gaussian random noise. $p_n$ denotes the dimension of coefficient parameter $\boldsymbol{\beta}$, and it can increase with the sample size $n$. 
The research objective is to make consistent variational Bayesian inferences on the coefficient $\boldsymbol{\beta}$.
Note that we are particularly interested in the high dimensional setting, i.e. $p_n \gg n$, but our developed theory and methodology hold for general dimensional setting. For the simplicity of analysis, $\sigma^2$ is assumed to be known throughout our theoretical analysis, while in practice it can be estimated by frequentist methods (see \cite{Reid2016Study} for a comprehensive review), Empirical Bayesian approach \citep{Castillo2015Bayesian}, 
% empirical Bayes method \textcolor{red}{add some reference} 
or full Bayesian analysis (for example, placing inverse gamma prior on $\sigma^2$ \citep{Ishwaran2005Spike, Park2009Bayesian}). 

Let $\boldsymbol{\beta}^0$ denote the true coefficient vector, and we assume that $\boldsymbol{\beta}^0$ has certain sparsity structure. The corresponding true model is denoted as $\xi^0=\{j:\beta_j\neq 0\}$, and true sparsity is denoted as $s = \|\boldsymbol{\beta}^0\|_0=|\xi^0|$, which is the cardinality of the true subset model. Note that $s$ is allowed to increase with $n$ as well. Let $\xi\subseteq\{1,\dots,p_n\}$ be the generic notation for any subset model, and $\boldsymbol{X}_\xi$ and $\boldsymbol{\beta}_\xi$ respectively denote the sub-matrix of $\boldsymbol{X}$ and sub-vector of $\boldsymbol{\beta}$ corresponding to $\xi$.

The following regularity conditions are required for the main results:
\begin{cond} \label{cond1}
The column norms of the design matrix are bounded by $n$, i.e.
$\sum_iX^2_{ij}=\|\boldsymbol{X}_j\|^2_2\leq n$.
\end{cond}
%\begin{cond}
%The dimensionality is high: $p \succeq n$.
%\end{cond}
\begin{cond} \label{cond2}
There exist some integer $\overline{p}$ (depending on $n$ and $p_n$) and fixed constant $\lambda_0$, such that $\overline{p} \succ s$\footnote{$a_n\prec b_n$ means $\lim_n a_n/b_n=0$.} and the smallest eigenvalue of $\boldsymbol{X}^T_{\xi}\boldsymbol{X}_{\xi}$ is greater than $n\lambda_0$ for any subset model $|\xi| \leq \overline{p}$.
\end{cond}
\begin{cond} \label{cond3}
$\log(\max_j|\beta_j^0|)=O(\log (p_n\vee n))$\footnote{$a\vee b$ denotes $\max(a,b)$.}.
\end{cond}

{\bf Remark:}
Condition \ref{cond1} is trivially satisfied when the covariates $X_{ij}$ are bounded by 1, or the design matrix is properly standardized. This bound condition is assumed for the technical simplicity, readers of interest can generalize this condition to that all covariates follow a sub-Gaussian distribution. Condition \ref{cond2} imposes a regularity assumption on the eigen structure of the design matrix which controls the multicollinearity. Similar conditions are commonly used in the literature of high dimensional statistics \citep{Zhang2010Nearly, Narisetty2014Bayesian, song2017nearly}. Under a random design scenario, if all entries of the design matrix are i.i.d. sub-Gaussian variables, then the random matrix theory \citep[e.g.,][]{Vershynin2012Introduction} guarantees that w.h.p., the eigen structure restriction holds with $\bar p$ being at least of order ${n/\log p_n}$, hence the condition $\overline{p} \succ s$ is met w.h.p. by assuming the common dimensionality condition $s\log p_n\prec n$. Condition \ref{cond3} imposes an upper bound for the magnitude of true coefficients, it allows the magnitude of $\boldsymbol{\beta}$ increases polynomially with respect to $p_n\vee n$. Similar bounded conditions on true coefficient are common among Bayesian theoretical literature \cite[e.g.,][]{Yang2016On}. Such conditions are necessary to ensure that the prior density around $\boldsymbol{\beta^0}$ is bounded away from zero, such that the domination of posterior around $\boldsymbol{\beta^0}$ becomes possible.

\subsection{Heavy Tail Shrinkage Prior and Variational Inference}
\paragraph{Prior Distribution}To resemble a spike-and-slab prior, a reasonable choice of shrinkage prior shall (1) allocate large probability mass around a small neighborhood of zero, i.e., a prior spike around 0; and (2) possess a very flat tail, i.e., a prior slab over real line. Following the suggestion by recent Bayesian literature \citep[e.g.,][]{song2017nearly,ghosh2015posterior,song2020bayesian}, our work will implement heavy-tailed prior distribution, i.e., polynomially decaying prior with properly tuning scale hyperparameter. For the simplicity of representation, this paper will only consider the theory and computation under student-$t$ prior, however, the general insights obtained apply to any heavy tailed priors.

%An example of such prior is the student-$t$ prior. For the remaining parts of this paper, we will only consider the student-$t$ prior, due to the fact that it will lead to algorithms with simple forms of updating.

Consider an independent $t$ prior for $\boldsymbol{\beta}$, which can be rewritten as a scaled mixture of Gaussian distribution with Inverse-Gamma scaling distributions, i.e., for $j = 1, \ldots, p_n$,
\begin{equation*}
\pi(\beta_j|\lambda_j) = \mathcal N(0, \lambda_j^{-1}), \quad \pi(\lambda_j) = Gamma(a_0, b_n).
\end{equation*}
where $a_0, b_n$ are user-specified hyperparameters. Thus, it yields a student-$t$ prior of d.f. $2a_0$ with scale parameter $\sqrt{b_n/a_0}$.
In other words, $a_0$ determines the polynomial degree of prior tail decay, i.e., the prior tail shape, while $b_n$ controls the scale of prior distribution.
As demonstrated by numerous Bayesian results \citep[e.g.,][]{song2020bayesian, van2014horseshoe,van2016conditions}, the prior scale needs to converge to zero as dimensionality increases,
hence we let $a_0$ be a constant, and $b_n$ asymptotically decrease as $n$ increases.

\paragraph{Variational Inference} Denote the posterior distribution of $\boldsymbol{\beta}$ as $\pi(\boldsymbol{\beta}|\boldsymbol{X},\boldsymbol{Y})$, then the variational inference \citep{Jordan1999Variational, Blei2017variational} seeks to find the best approximate distribution $\widehat q(\boldsymbol{\beta})$ from a given family of distributions $\mathcal{Q}$ that minimizes the 
Kullback-Leibler (KL) divergence $\mbox{KL}(q(\boldsymbol{\beta})\|\pi(\boldsymbol{\beta}|\boldsymbol{X},\boldsymbol{Y})))$. If we denote
\begin{equation} \label{eq:nelbo}
\Omega = -\mathbb{E}_{q(\boldsymbol{\beta})}[\log p(\boldsymbol{Y}|\boldsymbol{\beta})] + \mbox{KL}(q(\boldsymbol{\beta})\|\pi(\boldsymbol{\beta})),
\end{equation}
where $\Omega$ is the so-called negative ELBO, then the variational posterior is equivalent to 
\[
\widehat{q}(\boldsymbol{\beta}) = \arg\min_{q(\boldsymbol{\beta})\in \mathcal{Q}} \Omega.
\]
Note that the objective function (\ref{eq:nelbo}) can be viewed as the penalized negative mean log-likelihood function, with a regularization in term of the KL divergence between posterior candidate $q(\boldsymbol{\beta})$ and prior $\pi(\boldsymbol{\beta})$. 
%{\bf Remark}
%It is meaningful to view the first component and second component of (\ref{eq:nelbo}) as the averaged log-likelihood (w.r.t. variational distribution $q(\boldsymbol{\beta})$) and the regularization in terms of KL divergence, repectively.

In general, the variational family $\mathcal{Q}$ can be chosen freely and a common choice is the mean-field family. In this work, we choose $\mathcal{Q}$ as  independent student-$t$ distribution to resemble the prior distribution, i.e.
\[
q(\beta_j|\lambda_j) = \mathcal{N}(\mu_j, \lambda_j^{-1}), \quad q(\lambda_j)=Gamma(a_j, b_j),
\]
where $\mu_j \in \mathbb{R}$, $a_j >0$, $b_j >0$ for $j = 1, \ldots, p_n$.

{\bf Remark:} Choosing a different heavy tailed distribution as the prior distribution (e.g., horseshoe prior \cite{carvalho2009handling}) and variational family $\mathcal{Q}$ doesn't hurt the validity of the consistency result displayed in the next section, except that we require a different condition on the prior shape and  scale hyperparameters. However, the difficulty of minimizing (\ref{eq:nelbo}) varies from case to case, depending on the existence of closed form for the negative ELBO.

\section{THEORETICAL RESULTS}

To establish consistency of variational Bayes posterior, we impose the following condition on the prior specification.
\begin{cond} \label{cond:bn}
$a_0 > 1$ and $(p_n\vee n)^{-K}\prec b_n/a_0 \prec s\log (p_n\vee n)/[np_n^{2+1/a_0}(p_n\vee n)^{\delta/a_0}]$ for some large constant $K$ and small constant $\delta > 0$.
\end{cond}
$a_0>1$ ensures the existence of the second moment for the prior distribution, and the scale $b_n$ is required to decrease polynomially w.r.t $n$ and $p_n$, such that the prior contains a steep spike at 0. 
%\subsection

First, we study the infimum of the negative ELBO $\Omega$ (up to a constant). Define the loglikelihood ratio as
\[
l_n(P_0, P_{\boldsymbol{\beta}}) = \log \frac{p(\boldsymbol{Y}|\boldsymbol{\beta^0})}{p(\boldsymbol{Y}|\boldsymbol{\beta})}= \sum^n_{i=1}\log \frac{p(Y_i|\boldsymbol{\beta^0})}{p(Y_i|\boldsymbol{\beta})},
\]
then we have the following theorem.

\begin{thm} \label{thm:elbo}
With dominating probability for some $C >0$, we have
\begin{equation} \label{eq:elbo}
\begin{split}
   \inf_{q(\boldsymbol{\beta}) \in \mathcal{Q}}\Bigl\{\mbox{KL}(q(\boldsymbol{\beta})\|\pi(\boldsymbol{\beta}))+\int l_n(P_0, P_{\boldsymbol{\beta}})q(\boldsymbol{\beta})d\boldsymbol{\beta}\Bigr\}\leq Cs\log(p_n\vee n).
\end{split}
\end{equation}
\end{thm}
{\bf Remark:}
Theorem $\ref{thm:elbo}$ establishes the upper bound of the loss function corresponding to the variational posterior.

%\subsection{Variational Contraction Rate}

Our next theorem studies how fast the variational posterior contrasts toward the true $\boldsymbol{\beta}^0$.
\begin{thm} \label{thm:contraction}
With dominating probability, for any slowly diverging sequence $M_n'$, we have
\[
\widehat{q}(\|\boldsymbol{\beta} - \boldsymbol{\beta}^0\|_2 \geq M_n'\sqrt{s\log(p_n\vee n)/n}) = o(1).
\]
\end{thm}

{\bf Remark:}
Theorem \ref{thm:contraction} implies that the contraction rate of the variational posterior $\widehat q(\boldsymbol{\beta})$ is of order $\sqrt{s\log(p_n\vee n)/n}$. Under low dimensional setting, it reduces to $\sqrt{s/n}\log^{0.5}(n)$ which is the optimal rate up to a logarithmic term;
Under high dimensional setting, it reduces to $\sqrt{s\log(p_n)/n}$ which is the near-optimal convergence rate\footnote{The optimal is of order $\sqrt{s\log(p_n/s)/n}$.} commonly achieved in the literature.
In other words, there is little loss in term of distributional convergence asymptotics by implementing variational approximation. The variational inference procedure delivers consistent Bayesian inferences.

\section{IMPLEMENTATION}

\subsection{Updating Equations}

%\textcolor{red}{mention we minimize the joint KL of beta and lambda. I ask you to do some toy simulation, in order to show minimizing joint KL of beta and lambda lead to similar result of minimizing the marginal beta kl. Did you do it?}

%\textcolor{blue}{I haven't done it since I will need to write python code for it. I will add it later}
The direct optimization of the negative ELBO (\ref{eq:nelbo}) requires stochastic gradient descent algorithm, since there is no closed form for the KL divergence between two student-$t$ distributions. Therefore, for the purpose of efficient optimization, we instead consider minimizing the KL divergence of the joint distribution of $\boldsymbol{\beta}$ and $\boldsymbol{\lambda}$, where the negative ELBO is defined as
\begin{equation} \label{eq:loss}
\begin{split}
 \Omega =-\int \log p(\boldsymbol{Y}|\boldsymbol{\beta}, \lambda)q(\boldsymbol{\beta}|\boldsymbol{\lambda})q(\boldsymbol{\lambda})d\boldsymbol{\beta} d\boldsymbol{\lambda} +
\int \mbox{KL}(q(\boldsymbol{\beta}|\boldsymbol{\lambda})\|\pi(\boldsymbol{\beta}|\boldsymbol{\lambda})) q(\boldsymbol{\lambda})d\boldsymbol{\lambda} + \mbox{KL}(q(\boldsymbol{\lambda})\|\pi(\boldsymbol{\lambda})).
\end{split}
\end{equation}
{As showed by the toy examples in the appendix, the variational inference results derived based on minimizing the KL divergence of the joint distribution of $(\boldsymbol{\beta},\boldsymbol{\lambda})$ has little difference to the ones based on minimizing the KL divergence of marginal distribution of $\boldsymbol{\beta}$.}

We minimize (\ref{eq:loss}) by iteratively updating variational parameters in the fashion of coordinate descent. The updating equations are provided in below, where detailed derivation can be found in the appendix.

\paragraph{Updating $\boldsymbol{\mu}$} By fixing $a_j's$ and $b_j's$, the mean vector $\boldsymbol{\mu}=(\mu_1, \ldots, \mu_{p_n})^T$ is updated by 
\begin{equation} \label{eq:mu}
\boldsymbol{\mu} =  (\boldsymbol{X}^T\boldsymbol{X}+\sigma^2\boldsymbol{\Lambda})^{-1}\boldsymbol{X}^T\boldsymbol{Y},
\end{equation}
where $\boldsymbol{\Lambda}=\mbox{diag}(a_1/b_1,\dots,a_{p_n}/b_{p_n})$.

\paragraph{Updating $a_j$} By fixing $\mu_j$ and $b_j$, $a_j$ is updated by solving the following equation
\begin{equation} \label{eq:a}
-\frac{n_j}{2\sigma^2}\frac{b_j}{(a_j-1)^2} + \frac{\mu^2_j/2+b_n}{b_j} + (a_j - a_0)\psi_1(a_j) - 1=0.
\end{equation}

\paragraph{Updating $b_j$} By fixing $\mu_j$ and $a_j$, $b_j$ is updated by 
\begin{equation} \label{eq:b}
b_j = \frac{-a_0+\sqrt{a^2_0+2n_ja_j(\mu^2_j/2 + b_n)/\sigma^2(a_j-1)}}{n_j/\sigma^2(a_j-1)}.
\end{equation}

\subsection{Computation for Large $p_n$}
The major computational bottleneck of the above updating rule is the inversion of the large $p_n \times p_n$ matrix $(\boldsymbol{X}^T\boldsymbol{X}+\sigma^2\boldsymbol{\Lambda})$ in (\ref{eq:mu}), which could lead to huge computation cost. 

Instead, (\ref{eq:mu}) could be improved by using the blockwise update strategy introduced by \cite{Ishwaran2005Spike}. Specifically, decompose $\boldsymbol{\mu}$ as $(\boldsymbol{\mu}_{(1)}, \ldots, \boldsymbol{\mu}_{(B)})^T$, $\boldsymbol{\Lambda}$ as $\mbox{diag}(\boldsymbol{\Lambda}_{(1)}, \ldots, \boldsymbol{\Lambda}_{(B)})$ and $\boldsymbol{X}$ as $[\boldsymbol{X}_{(1)}, \ldots, \boldsymbol{X}_{(B)}]$, where $B$ is the number of blocks. Denote the exclusion of the $k$th block using subscript $(-k)$, then the blockwise update for $\boldsymbol{\mu}$ is
\begin{equation} \label{eq:mu_block}
    \boldsymbol{\mu}_{(k)}=(\boldsymbol{X}^T_{(k)}\boldsymbol{X}_{(k)}+\sigma^2\boldsymbol{\Lambda}_{(k)})^{-1}\boldsymbol{X}_{(k)}^T(\boldsymbol{Y}-\boldsymbol{X}_{(-k)}\boldsymbol{\mu}_{(-k)}),
\end{equation}
for $k = 1, \ldots, B$. The blockwise update will reduce the order of computational complexity from $O(p_n^3)$ to $O(B^{-2}p_n^3)$ \citep{Ishwaran2005Spike}, which could alleviate the computation burden when $p_n$ is huge.

To summarize, the variational inference with Student-$t$ prior is shown in Algorithm \ref{alg:vb}.

\begin{algorithm}
\caption{Variational inference with Student-$t$ prior.} \label{alg:vb}
\begin{algorithmic}[1]
\State {Hyperparameters: $a_0$, $b_n$}
\State {\textbf{Initialize} $\boldsymbol{\mu}, \{a_j\}^{p_n}_{j=1}, \{b_j\}^{p_n}_{j=1}$} 
\Repeat
    \For{$k = 1 \mbox{ to }B $}
        \State {$\boldsymbol{\mu}_{(k)}$ $\gets$ apply equation  (\ref{eq:mu_block})}
    \EndFor
    
    \FORALLP {$j \in \{1, \ldots, p_n\}$}
    \State {$a_j$ $\gets$ solve equation (\ref{eq:a})} 
    \State {$b_j$ $\gets$ apply equation (\ref{eq:b}) } 
    \ENDFAP
    \State {$\Omega$ $\gets$ $\boldsymbol{\mu}, \{a_j\}^{p_n}_{j=1}, \{b_j\}^{p_n}_{j=1}$ using (\ref{eq:loss})}
\Until{convergence of $\Omega$}
\State \Return {$\boldsymbol{\mu}, \{a_j\}^{p_n}_{j=1}, \{b_j\}^{p_n}_{j=1}$}
\end{algorithmic}
\end{algorithm}

Note that it is crucial that the algorithm allows us to update the key variational parameter $\boldsymbol\mu$ blockwisely. Comparing to Algorithm 1 of \cite{Ray2020Variational} which has to update variational parameter entrywisely, our algorithm has a much better convergence speed. In addition, Algorithm 1 of \cite{Ray2020Variational} also has to conduct more iterations of univariate numerical optimizations. Therefore, as showed by our simulation studies, our algorithm has much faster computing speed.

\section{NUMERICAL STUDIES}
\begin{table*}[b!]
\centering
\footnotesize
\caption{Regression Results for Example 1 (a): Strong Signal Case}
\label{tb:exp1a}
\begin{tabular}{lcccccc}
\toprule
         & \multicolumn{4}{c}{Bayesian} & \multicolumn{2}{c}{Non-Bayesian}                          \\ \cmidrule(lr){2-5} \cmidrule(lr){6-7} 
         & t-VB &t-MCMC& Laplace & varbvs & SSLASSO & EMVS  \\ \midrule
RMSE     & 0.29 $\pm$ 0.18 &0.21 $\pm$ 0.20  & 0.23 $\pm$ 0.05        & 0.66 $\pm$ 0.38  & 0.13 $\pm$ 0.11 & 0.88 $\pm$ 0.04   \\
FDR      & 0.18 $\pm$ 0.15 &0.02 $\pm$ 0.09  & 0.01 $\pm$ 0.05       & 0.09 $\pm$ 0.17  & 0.01 $\pm$ 0.08 & 0.04 $\pm$ 0.11   \\
TPR      & 0.96 $\pm$ 0.11 &0.95 $\pm$ 0.13  & 0.99 $\pm$ 0.02       & 0.42 $\pm$ 0.41  & 0.99 $\pm$ 0.07 & 0.17 $\pm$ 0.07    \\
Coverage of $\xi^0$&0.88 $\pm$ 0.03&0.81 $\pm$ 0.03 & 0.87 $\pm$ 0.04 & 0.23 $\pm$ 0.04& -&-\\
Coverage of $(\xi^0)^c$ & 0.99 $\pm$ 0.01 & 0.99 $\pm$ 0.00& 0.99 $\pm$ 0.00& 0.99 $\pm$ 0.00&-&-\\
Run time & 0.54 $\pm$ 0.05 &23.71 $\pm$ 0.71  & 7.65 $\pm$ 2.65        & 0.29 $\pm$ 0.11  & 0.06 $\pm$ 0.01 & 0.35 $\pm$ 0.06   \\ \bottomrule
\end{tabular}
\end{table*}

In this section, we validate the effectiveness of our method via simulation experiments. To satisfy Condition \ref{cond:bn}, throughout this section, we let $a_0=2$ and $b_n/a_0 =\log (p_n\vee n)/[np_n^{2+1/a_0}(p_n\vee n)^{1/a_0}]$.
% \textcolor{red}{condition \ref{cond:bn} involves unknown $s$? okay, I think you can replace s by 1.}. 
% \textcolor{blue}{sure}
We use Lasso estimator to initialize $\boldsymbol{\mu}$. $a_j$ and $b_j$ are initialized as $(a_0+0.5)$ and $(b_n + \mu_j^2)$ respectively. The following rule is used to derive variable selection results: if the $95\%$ credible interval of marginal $t$ variational posterior contains 0, then the corresponding predictor is not selected, and vice versa. This method of Bayesian model selection under shrinkage priors is discussed by \citep{van2017uncertainty}. More sophisticated approaches under variational Bayesian shrinkage for model selection could be a future study direction.

Both variational inference ($t$-VB) and MCMC ($t$-MCMC) are implemented under the same student-$t$ prior for fair comparison, where $t$-MCMC is computed by Gibbs sampler \citep{song2017nearly}. We also compare our method to the following competitive methods: variational Bayes for spike-and-slab priors with Laplace slabs (Laplace) \citep{Ray2020Variational}, variational Bayes for spike-and-slab priors with Gaussian slabs (varbvs) \citep{Carbonetto2012Scalable},  the spike-and-slab LASSO (SSLASSO) \citep{Rockova2018spike}, and the EM algorithm for spike-and-slab prior (EMVS) \citep{Rockova2014EMVS}. 
% and the fractional likelihood empirical Bayes approach using MCMC for re-centered Gaussian slab priors (ebreg) \citep{Martin2017empirical}.

For $t$-MCMC, we run Gibbs update for 1000 iterations with 200 burning in, and the initialization is the same as $t$-VB. We employ the blocklization \citep{Ishwaran2005Spike} for the Gibbs update. For Laplace, we use hyper-parameter $a_0=1, b_0=n$ and $\lambda=1$. The ridge estimator $(\boldsymbol{X}^T\boldsymbol{X}+\boldsymbol{I})^{-1}\boldsymbol{X}^T\boldsymbol{Y}$ is used for initialization and the unknown $\sigma$ is estimated by selectiveInference package \citep{Reid2016Study}. For other methods, we use their associated R packages with default parameters. All the methods are implemented on the MacBook Pro with 2.7 GHz Intel Core i7.

The metrics reported are the Root Mean Squared Error between the posterior mean estimator $\widehat{\boldsymbol{\beta}}$ and $\boldsymbol{\beta}^0$ (RMSE), the False Discovery Rate (FDR), True Postitive Rate (TPR), and the run time. For Bayesian methods, the Coverage rates of $95 \%$ credible intervals for non-zero coefficients $\xi^0$ and zero coefficients $(\xi^0)^c$ are also calculated. All the experiments are repeated 100 times and the mean metric together with its standard deviation are reported.

\subsection{Example 1: Moderate Dimension Case}
This is an example similar to the one in \cite{Ray2020Variational}. Let $n=100$, $p_n=400$ and $s=20$. All the nonzero coefficients are equal to $\log(n)$ (strong) or $\log(n)/2$ (weak) and their positions are randomly located within the $p_n$ dimension coefficient vector. Take the design matrix $X_{ij} \overset{iid}{\sim} \mathcal{N}(0, 1)$ and assume $\sigma$ is known that equals 4. Since $p_n$ is moderate, we choose $B=1$ when update $\boldsymbol{\mu}$ and use 5 blocks for Gibbs update. 

Table \ref{tb:exp1a} shows for relatively large signal, SSLASSO achieves the best estimation accuracy and the smallest selection error with the shortest run time, however it can not give second-order inferences. Among Bayesian methods, our method achieves estimation accuracy close to that of MCMC with the shortest run time. Meanwhile, the variable selection errors and the coverage rates of our method are also close to those of MCMC. Table \ref{tb:exp1b} exhibits when the signal is relatively weak, our method obtains the estimation accurracy and selection error close to the best ones (Laplace) with much shorter run time. The MCMC is underperformed in this case probably due to insufficient number of Gibbs iterations. Note that the FDR for EMVS is undefined since none of the predictors is selected.

\begin{table*}[t!]
\centering
\footnotesize
\caption{Regression Results for Example 1 (b): Weak Signal Case.}
\label{tb:exp1b}
\begin{tabular}{lcccccc}
\toprule
         & \multicolumn{4}{c}{Bayesian} & \multicolumn{2}{c}{Non-Bayesian}                          \\ \cmidrule(lr){2-5} \cmidrule(lr){6-7}          & t-VB &t-MCMC & Laplace & varbvs & SSLASSO & EMVS \\ \midrule
RMSE     & 0.38 $\pm$ 0.07&0.46 $\pm$ 0.78  &0.36 $\pm$ 0.07  & 0.45 $\pm$ 0.07   &0.42 $\pm$ 0.10 &0.46 $\pm$ 0.01   \\
FDR      & 0.29 $\pm$ 0.15&0.12 $\pm$ 0.16 &0.16 $\pm$ 0.15  & 0.12 $\pm$ 0.21   &0.26 $\pm$ 0.29 & -    \\
TPR      & 0.57 $\pm$ 0.20&0.27 $\pm$ 0.10  &0.62 $\pm$ 0.18  & 0.19 $\pm$ 0.22   &0.47 $\pm$ 0.18 & 0.00 $\pm$ 0.00   \\
Coverage of $\xi^0$ & 0.40 $\pm$ 0.05& 0.22 $\pm$ 0.03&0.54 $\pm$ 0.04 &0.14 $\pm$ 0.03&-&-\\
Coverage of $(\xi^0)^c$ &0.99 $\pm$ 0.01   & 0.99 $\pm$ 0.00&0.99 $\pm$ 0.01 & 0.99 $\pm$ 0.00&-&-\\
Run time & 0.54 $\pm$ 0.04& 23.52 $\pm$ 0.38  &13.78 $\pm$ 8.71 & 0.29 $\pm$ 0.13   &0.07 $\pm$ 0.02 & 0.18 $\pm$ 0.02  \\ \bottomrule
\end{tabular}
\end{table*}

\subsection{Example 2: High Dimension Case}
We consider an example similar to the one in \cite{Rockova2014EMVS}. Let $n=100$, $p=1000$ and $\boldsymbol{\beta}^0=(3,2,1,0,\ldots,0)^T$. Generate the design matrix $X_{ij} \overset{iid}{\sim} \mathcal{N}(0, 1)$. Assume $\sigma=1$ and it is unknown in the experiment. For our method, we use the Empirical Bayes estimator for $\sigma$ here. Specifically, by optimizing $\Omega$ w.r.t. $\sigma$, the Empirical Bayes (EB) update of $\sigma$ follows
\begin{equation} \label{eq:sigma}
\sigma = \sqrt{\frac{(\boldsymbol{Y}-\boldsymbol{X}\boldsymbol{\mu})^T(\boldsymbol{Y}-\boldsymbol{X}\boldsymbol{\mu})+\sum^{p_n}_{j=1}\frac{n_jb_j}{a_j-1}}{n}}.
\end{equation}

Due to the high dimensionality, we choose $B=10$ when updating $\boldsymbol{\mu}$ and also use 10 blocks for Gibbs update. The results are reported in Table \ref{tb:exp2}.  

Table \ref{tb:exp2} shows all the methods achieve good estimation accuracies expect Laplace and EMVS. Our method also obtains similar selection errors and coverage rates to those of MCMC with much shorter time. Again, the FDR for EMVS is undefined since all the estimated coefficients are not selected.

% \begin{table*}[h!]
% \centering
% \caption{Example 2}
% \label{tb:exp1}
% \begin{tabular}{lcccc}
% \toprule
% Method    & $l_2$ error & FDR & TPR & Run Time \\ \midrule
% Student t & 0.32 $\pm$ 0.10    & 0.00 $\pm$ 0.04   & 1.00 $\pm$ 0.00    & 0.48 $\pm$ 0.05      \\
% Laplace   &                    &                   &                    &                      \\
% varbvs    & 0.18 $\pm$ 0.10    & 0.02 $\pm$ 0.07   & 1.00 $\pm$ 0.00    & 0.32 $\pm$ 0.11          \\
% SSLASSO   & 0.17 $\pm$ 0.07    & 0.00 $\pm$ 0.00   & 1.00 $\pm$ 0.00    & 0.75 $\pm$ 0.10 \\
% EMVS      & 3.57 $\pm$ 0.02    & -                 & 0.00 $\pm$ 0.00    & 0.17 $\pm$ 0.01 \\
% ebreg     & 0.20 $\pm$ 0.10    & 0.01 $\pm$0.05    & 0.99 $\pm$ 0.03    & 11.56 $\pm$ 2.21         \\ \bottomrule
% \end{tabular}
% \end{table*}
\begin{table*}[t!]
\centering
\footnotesize
\caption{Regression Results for Example 2.}
\label{tb:exp2}
\begin{tabular}{lcccccc}
\toprule
         & \multicolumn{4}{c}{Bayesian} & \multicolumn{2}{c}{Non-Bayesian}                          \\ \cmidrule(lr){2-5} \cmidrule(lr){6-7}          & t-VB &t-MCMC & Laplace & varbvs & SSLASSO & EMVS  \\ \midrule
RMSE      & 0.01 $\pm$ 0.00 &0.01 $\pm$ 0.00&0.06 $\pm$ 0.00         & 0.01 $\pm$ 0.00 & 0.01 $\pm$ 0.00 & 0.11 $\pm$ 0.00   \\
FDR      & 0.00 $\pm$ 0.04 &0.00 $\pm$ 0.00& 0.07 $\pm$ 0.16          & 0.02 $\pm$ 0.07 & 0.00 $\pm$ 0.00 & -               \\
TPR      & 1.00 $\pm$ 0.00 &1.00 $\pm$ 0.00 & 1.00 $\pm$ 0.00        & 1.00 $\pm$ 0.00 & 1.00 $\pm$ 0.00 & 0.00 $\pm$ 0.00  \\
Coverage of $\xi^0$ &0.99 $\pm$ 0.01 &0.95 $\pm$ 0.03 &0.90 $\pm$ 0.03 &0.16 $\pm$ 0.02&-&-\\
Coverage of $(\xi^0)^c$ & 1.00 $\pm$ 0.00  &1.00 $\pm$ 0.00 &0.99 $\pm$ 0.00 &0.99 $\pm$ 0.00 &-&-\\
Run time & 0.48 $\pm$ 0.05 &94.32 $\pm$ 3.40& 37.61 $\pm$ 23.55         & 0.32 $\pm$ 0.11 & 0.75 $\pm$ 0.10 & 0.17 $\pm$ 0.01  \\ \bottomrule
\end{tabular}
\end{table*}

\section{DISCUSSION}
We proposed a scalable variational inference algorithm for high dimensional linear regression under shrinkage priors. The established theoretical properties are justified by empirical studies. A possible future direction is to explore and compare efficient implementation for variational inference with other heavy tail shrinkage priors besides the Student-$t$.

\bibliography{ref}

\newpage
\begin{center}
	\textbf{\LARGE APPENDIX}
\end{center}
\appendix
\vspace{30px}

\section{DETAILED PROOFS}
{\noindent \bf Proof of Theorem \ref{thm:elbo}}
\begin{proof}
The marginal prior distribution for $\beta_j$ is
\[
\pi(\beta_j) = \frac{1}{\sqrt{\nu_0}s_0}\Bigl(1 + \nu_0^{-1}\Bigl(\frac{\beta_j}{s_0}\Bigr)^2 \Bigr)^{-\frac{\nu_0+1}{2}},
\]
where $s_0 = \sqrt{b_n/a_0}$ and $\nu_0 = 2a_0$. We define $q^*(\beta_j)$ as follows
\[
q^*(\beta_j) = \frac{1}{\sqrt{\nu^*}s^*}\Bigl(1 + (\nu^*)^{-1}\Bigl(\frac{\beta_j - \beta^0_j}{s_j}\Bigr)^2 \Bigr)^{-\frac{\nu^*+1}{2}},
\]
where $s^*=\sqrt{b_n/a_0}$ and $\nu^* = 2a_0$, and it is sufficient to show that $\mbox{KL}(q^*(\boldsymbol{\beta})\|\pi(\boldsymbol{\beta}))+\int l_n(P_0, P_{\boldsymbol{\beta}})q^*(\boldsymbol{\beta})d\boldsymbol{\beta} \leq Cs\log(p_n\vee n).$

i) We first show 
\begin{equation} \label{eq:likelihood}
    \int l_n(P_0, P_{\boldsymbol{\beta}})q^*(\boldsymbol{\beta})d\boldsymbol{\beta} \leq C_1s\log(p_n\vee n).,
\end{equation}
for some $C_1>0$.
Note that
\[
\begin{split}
&l_n(P_0, P_{\boldsymbol{\beta}}) = \frac{1}{2\sigma^2}(\|\boldsymbol{Y} - \boldsymbol{X}\boldsymbol{\beta}\|^2_2 
- \|\boldsymbol{Y} - \boldsymbol{X}\boldsymbol{\beta}^0\|^2_2)\\
= &\frac{1}{2\sigma^2} (\|\boldsymbol{Y} - \boldsymbol{X}\boldsymbol{\beta}^0+ \boldsymbol{X}\boldsymbol{\beta}^0
-\boldsymbol{X}\boldsymbol{\beta})\|^2_2
-\|\boldsymbol{Y} - \boldsymbol{X}\boldsymbol{\beta}^0\|^2_2)\\
= &\frac{1}{2\sigma^2}(\|\boldsymbol{X}\boldsymbol{\beta}-\boldsymbol{X}\boldsymbol{\beta}^0\|^2_2 
+ 2\langle \boldsymbol{Y}-\boldsymbol{X}\boldsymbol{\beta}^0, \boldsymbol{X}\boldsymbol{\beta}^0 - \boldsymbol{X}\boldsymbol{\beta}\rangle).
\end{split}
\]
Denote
\[
\begin{split}
\mathcal{R}_1 &= \int \|\boldsymbol{X}\boldsymbol{\beta} - \boldsymbol{X}\boldsymbol{\beta}^0\|^2_2 q^*(\boldsymbol{\beta})d\boldsymbol{\beta}, \\
\mathcal{R}_2 &=\int \langle \boldsymbol{Y}-\boldsymbol{X}\boldsymbol{\beta}^0, \boldsymbol{X}\boldsymbol{\beta}^0 - \boldsymbol{X}\boldsymbol{\beta})\rangle q^*(\boldsymbol{\beta})d\boldsymbol{\beta}.
\end{split}
\]
Noting that $\boldsymbol{Y} - \boldsymbol{X}\boldsymbol{\beta}^0= \sigma\boldsymbol{\epsilon} \sim \mathcal{N}(0, \sigma^2I_n)$, then
\[
\begin{split}
\mathcal{R}_2 &= \int \sigma\boldsymbol{\epsilon}^T(\boldsymbol{X}\boldsymbol{\beta}^0 - \boldsymbol{X}\boldsymbol{\beta})q^*(\boldsymbol{\beta})d\boldsymbol{\beta}
\\
&= \sigma\boldsymbol{\epsilon}^T\int (\boldsymbol{X}\boldsymbol{\beta}^0 - \boldsymbol{X}\boldsymbol{\beta})q^*(\boldsymbol{\beta})d\boldsymbol{\beta}
\sim \mathcal{N}(0, c_f\sigma^2),
\end{split}
\]
where $c_f = \|\int (\boldsymbol{X}\boldsymbol{\beta}^0 - \boldsymbol{X}\boldsymbol{\beta})q^*(\boldsymbol{\beta})d\boldsymbol{\beta}\|^2_2 \leq \mathcal{R}_1$ due to Cauchy-Schwarz inequality. Then by Gaussian tail bound
\[
P_0(\mathcal{R}_2 \geq \mathcal{R}_1) \leq \exp(\frac{\mathcal{R}^2_1}{2\sigma^2\mathcal{R}_1}),
\]
which implies $\mathcal{R}_2 \leq \mathcal{R}_1$ w.h.p.. Therefore, to prove (\ref{eq:likelihood})
it suffices to establish that $\mathcal{R}_1=O(s\log(p_n\vee n))$. Note that
\[
\int \|\boldsymbol{X}\boldsymbol{\beta} - \boldsymbol{X}\boldsymbol{\beta}^0\|^2_2 q^*(\boldsymbol{\beta})d\boldsymbol{\beta} \leq \|\boldsymbol{X}\|^2_2 \int \|\boldsymbol{\beta}-\boldsymbol{\beta}^0\|^2_2q^*(\boldsymbol{\beta})d\boldsymbol{\beta}.
\]
where $\|\boldsymbol{X}\|_2$ is the spectral norm of matrix $\boldsymbol X$.  Since $\|\boldsymbol{X}\|^2_2 \leq tr(\boldsymbol{X}^T\boldsymbol{X}) = np_n$, and 
\[
\int \|\boldsymbol{\beta}-\boldsymbol{\beta}^0\|^2_2q^*(\boldsymbol{\beta})d\boldsymbol{\beta} = \sum^{p_n}_{j=1}s^{*2} \frac{\nu^*}{\nu^*-2} = p_n \frac{b_n}{a_0-1},
\]
then
\[
\int \|\boldsymbol{X}\boldsymbol{\beta} - \boldsymbol{X}\boldsymbol{\beta}^0\|^2_2 q^*(\boldsymbol{\beta})d\boldsymbol{\beta} \leq np_n^2\frac{b_n}{a_0-1} =O(s\log(p_n\vee n)).
\]
for sufficiently large $n$.

ii) We next show
\begin{equation} \label{eq:KL}
    \mbox{KL}(q^*(\boldsymbol{\beta})\|\pi(\boldsymbol{\beta})) \leq C_2s\log(p_n\vee n),
\end{equation}
for some $C_2>0$.

Note that
\[
\begin{split}
&\mbox{KL}(q^*(\boldsymbol{\beta})\|\pi(\boldsymbol{\beta}))=\sum_{j=1}^{p_n} \mbox{KL}(q^*({\beta_j})\|\pi({\beta_j})) \\
&=\sum_{j:\beta_j^0\neq 0} \mbox{KL}(q^*({\beta_j})\|\pi({\beta_j})).
\end{split}
\]
For each $j$,
\[
\begin{split}
&\mbox{KL}(q^*(\beta_j)\|\pi(\beta_j)) \\
= &\frac{\nu^*+1}{2}\int \log \frac{\nu^*s^{*2}+\beta_j^2}{\nu^*s^{*2}+(\beta_j - \beta^0_j)^2} q^*(\beta_j)d\beta_j.
\end{split}
\]

If $\beta^0_j > 0$, then $\frac{\nu^*s^{*2}+\beta_j^2}{\nu^*s^{*2}+(\beta_j - \beta^0_j)^2}$ is maximized at $\widehat{\beta}_j = \frac{\beta^0_j + \sqrt{(\beta^0_j)^2+4\nu^*s^{*2}}}{2}$, and the maximum is
\[
\begin{split}
&\frac{\nu^*s^{*2}+\widehat{\beta}_j^2}{\nu^*s^{*2}+(\widehat{\beta}_j - \beta^0_j)^2} 
= \frac{(\beta^0_j)^2 + \beta^0_j\sqrt{(\beta^0_j)^2+4\nu^*s^{*2}}+4\nu^*s^{*2}}{(\beta^0_j)^2 - \beta^0_j\sqrt{(\beta^0_j)^2+4\nu^*s^{*2}}+4\nu^*s^{*2}}\\
\leq& \frac{(\beta^0_j)^2 + \beta^0_j\sqrt{(\beta^0_j)^2+4\nu^*s^{*2}}+4\nu^*s^{*2}}{
(\beta^0_j)^2 - \beta^0_j(\beta_j^0+\frac{4\nu^*s_0^2}{2\beta_j^0})+4\nu^*s^{*2}}
= \frac{(\beta^0_j)^2 + \beta^0_j\sqrt{(\beta^0_j)^2+4\nu^*s^{*2}}+4\nu^*s^{*2}}{2\nu^*s^{*2}}.
\end{split}
\]
Therefore, for sufficiently large n, 
\[
\begin{split}
\mbox{KL}(q^*(\beta_j)\|\pi(\beta_j)) &\leq  \frac{\nu^*+1}{2}\times O(\log (\beta_j^0/s^{*}))
=O(\log(p_n\vee n)).
\end{split}
\]
Similar result holds if $\beta^0_j < 0$ as well. This imples that $ \mbox{KL}(q^*(\boldsymbol{\beta})\|\pi(\boldsymbol{\beta})) =O(s_n\log(p_n\vee n))$, and hence verifies (\ref{eq:KL}).
%2) If $\beta^0_j < 0$, then $\frac{\nu^*s^{*2}+\beta_j^2}{\nu^*s^{*2}+(\beta_j - \beta^0_j)^2}$ is maximized at $\widehat{\beta}_j = \frac{\beta^0_j - \sqrt{(\beta^0_j)^2+4\nu^*s^{*2}}}{2}$, and the maximum is
%\[
%\begin{split}
%&\frac{\nu^*s^{*2}+\widehat{\beta}_j^2}{\nu^*s^{*2}+(\widehat{\beta}_j - \beta^0_j)^2} = \frac{(\beta^0_j)^2 - \beta^0_j\sqrt{(\beta^0_j)^2+4\nu^*s^{*2}}+4\nu^*s^{*2}}{(\beta^0_j)^2 + \beta^0_j\sqrt{(\beta^0_j)^2+4\nu^*s^{*2}}+4\nu^*s^{*2}}\\
%\leq &\frac{(\beta^0_j)^2 - \beta^0_j\sqrt{(\beta^0_j)^2+4\nu^*s^{*2}}+4\nu^*s^{*2}}{(\beta^0_j)^2 + \beta^0_j(2\sqrt{\nu^*}s^* - \beta^0_j)} \leq \frac{\beta^0_j - \sqrt{(\beta^0_j)^2+4\nu^*s^{*2}}+4\nu^*s^{*2}/\beta^0_j}{ 2\sqrt{\nu^*}s^*}\\
%\leq & \log p,
%\end{split}
%\]
%for sufficient large $n$. Noting that $\mbox{KL}(q^*(\beta_j)\|\pi(\beta_j))=0$ for $\beta^0_j$, 1) and 2) lead to (\ref{eq:KL}).
Therefore, (\ref{eq:elbo}) immediately follows from (\ref{eq:likelihood}) and (\ref{eq:KL}).
\end{proof}

The next lemma states the existence of testing condition. Define $\widetilde p$ as some sequence satisfying $s\leq \widetilde{p}\leq \bar p-s$, $\widetilde p\prec p_n$ and $\lim \widetilde p =\infty$.
Let $\varepsilon_n=\sqrt{\widetilde p\log(p_n\vee n)/n}$.
Denote $B_n$ as the truncated parameter space 
\[
B_n = \{\boldsymbol{\beta}: \mbox{ at most } \widetilde{p} \mbox{ entries of } |\boldsymbol{\beta}/\sigma| \mbox{ is larger than } a_n\}
\]
and 
\[
C_n = \{\boldsymbol{\beta}:  B_n \cap \{\|\boldsymbol{\beta} - \boldsymbol{\beta}^0\|_2 \geq  M_n\varepsilon_n\},
\]
where $a_n \asymp \sqrt{s\log (p_n\vee n)/n}/p_n$,  $M_n$ is any diverging sequence as $M_n\rightarrow\infty$.

\begin{lem} \label{lm:testing}
There exists some testing function $\phi_n \in [0, 1]$ and $c_1 >0$, $c_2 >1/3$, such that
\[
\begin{split}
    \mathbb{E}_{\boldsymbol{\beta}^0}\phi_n &\leq \exp(-c_1n\varepsilon^2_n) \\
    \sup_{\boldsymbol{\beta} \in C_n}\mathbb{E}_{\boldsymbol{\beta}}(1-\phi_n) &\leq \exp(-c_2nM_n^2\varepsilon^2_n)
\end{split}
\]
\end{lem}
\begin{proof}
The construction of the testing function is similar to that of \cite{song2017nearly}.
Consider the following testing function
\[
\phi_n = \max_{\{\xi \supset \xi^0, |\xi| \leq \widetilde{p}+s\}}1\{\|(\boldsymbol{X}^T_{\xi}\boldsymbol{X}_{\xi})^{-1}\boldsymbol{X}^T_{\xi}\boldsymbol{Y} - \boldsymbol{\beta}^0_{\xi}\|_2\geq \sigma M \varepsilon_n\}
\] for some constant $M$.

i) For any $\xi$, such that  $\xi\supset \xi^0, |\xi| \leq \widetilde{p}+s$, 
\[
\begin{split}
& \mathbb{E}_{\boldsymbol{\beta}^0}1\{\|(\boldsymbol{X}^T_{\xi}\boldsymbol{X}_{\xi})^{-1}\boldsymbol{X}^T_{\xi}\boldsymbol{Y} - \boldsymbol{\beta}^0_{\xi}\|_2\geq \sigma M\varepsilon_n\}
=\mathbb{E}_{\boldsymbol{\beta}^0}1\{\|(\boldsymbol{X}^T_{\xi}\boldsymbol{X}_{\xi})^{-1}\boldsymbol{X}^T_{\xi}\boldsymbol{\epsilon}\|_2\geq M\varepsilon_n\}\\
\leq &Pr(\|(\boldsymbol{X}^T_{\xi}\boldsymbol{X}_{\xi})^{-1}\|_2(\boldsymbol{\epsilon}^T\boldsymbol{H}_{\xi}\boldsymbol{\epsilon}) \geq M^2\varepsilon^2_n)
\leq  Pr(\chi^2_{|\xi|}\geq n\lambda_0 M^2\varepsilon^2_n)
\leq  \exp(-c_1'M^2n\varepsilon^2_n)
\end{split}
\]
for some constant $c'$, where $\boldsymbol{H}_{\xi} = \boldsymbol{X}_{\xi}(\boldsymbol{X}^T_{\xi}\boldsymbol{X}_{\xi})^{-1}\boldsymbol{X}^T_{\xi}$, and the last inequality is due to the sub-exponential properties of chi-square distribution and $|\xi|\ll n\epsilon_n^2$. This further implies that 
\[
\begin{split}
 \mathbb{E}_{\boldsymbol{\beta}^0}\phi_n&\leq \sum_{\{\xi \supset \xi^0, |\xi| \leq \widetilde{p}+s\}} \mathbb{E}_{\boldsymbol{\beta}^0}1\{\|(\boldsymbol{X}^T_{\xi}\boldsymbol{X}_{\xi})^{-1}\boldsymbol{X}^T_{\xi}\boldsymbol{Y} - \boldsymbol{\beta}^0_{\xi}\|_2\geq \sigma M\varepsilon_n\}\\
 &\leq p_n^{\widetilde p+s} \exp(-c_1'M^2n\varepsilon^2_n)\leq \exp(-c_1n\varepsilon^2_n)
 \end{split}
\]
when $M$ is sufficiently large.

ii) Let $\widetilde{\xi}=\{k:|\beta_k/\sigma| > a_n\} \cup \xi^0$, then
\[
\begin{split}
&\sup_{\boldsymbol{\beta} \in C_n}\mathbb{E}_{\boldsymbol{\beta}}(1-\phi_n) = \sup_{\boldsymbol{\beta} \in C_n}\mathbb{E}_{\boldsymbol{\beta}}\min_{|\xi|\leq \widetilde{p}+s}1\{\|(\boldsymbol{X}^T_{\xi}\boldsymbol{X}_{\xi})^{-1}\boldsymbol{X}^T_{\xi}\boldsymbol{Y} - \boldsymbol{\beta}^0_{\xi}\|_2\leq \sigma\varepsilon_n\}\\
\leq &\sup_{\boldsymbol{\beta} \in C_n}\mathbb{E}_{\boldsymbol{\beta}}1\{\|(\boldsymbol{X}^T_{\widetilde{\xi}}\boldsymbol{X}_{\widetilde{\xi}})^{-1}\boldsymbol{X}^T_{\widetilde{\xi}}\boldsymbol{Y} - \boldsymbol{\beta}^0_{\widetilde{\xi}}\|_2\leq \sigma\varepsilon_n\}\\
= &\sup_{\boldsymbol{\beta} \in C_n}Pr\{\|(\boldsymbol{X}^T_{\widetilde{\xi}}\boldsymbol{X}_{\widetilde{\xi}})^{-1}\boldsymbol{X}^T_{\widetilde{\xi}}\boldsymbol{Y} - \boldsymbol{\beta}^0_{\widetilde{\xi}}\|_2\leq \sigma\varepsilon_n\}\\
= &\sup_{\boldsymbol{\beta} \in C_n}Pr\{\|(\boldsymbol{X}^T_{\widetilde{\xi}}\boldsymbol{X}_{\widetilde{\xi}})^{-1}\boldsymbol{X}^T_{\widetilde{\xi}}\sigma\boldsymbol{\epsilon} + \boldsymbol{\beta}_{\widetilde{\xi}} + (\boldsymbol{X}^T_{\widetilde{\xi}}\boldsymbol{X}_{\widetilde{\xi}})^{-1}\boldsymbol{X}^T_{\widetilde{\xi}}\boldsymbol{X}_{\widetilde{\xi^c}}\boldsymbol{\beta}_{\widetilde{\xi}^c}- \boldsymbol{\beta}^0_{\widetilde{\xi}}\|_2\leq \sigma\varepsilon_n\}\\
\leq &\sup_{\boldsymbol{\beta} \in C_n}Pr\{\|(\boldsymbol{X}^T_{\widetilde{\xi}}\boldsymbol{X}_{\widetilde{\xi}})^{-1}\boldsymbol{X}^T_{\widetilde{\xi}}\boldsymbol{\epsilon}\|_2\geq (\|\boldsymbol{\beta}_{\widetilde{\xi}} - \boldsymbol{\beta}^0_{\widetilde{\xi}}\|_2 - \sigma\varepsilon_n - \|(\boldsymbol{X}^T_{\widetilde{\xi}}\boldsymbol{X}_{\widetilde{\xi}})^{-1}\boldsymbol{X}^T_{\widetilde{\xi}}\boldsymbol{X}_{\widetilde{\xi^c}}\boldsymbol{\beta}_{\widetilde{\xi}^c}\|_2)/\sigma\}.
\end{split}
\]
Note that $\|\boldsymbol{X}_{\widetilde{\xi^c}}\boldsymbol{\beta}_{\widetilde{\xi}^c}\|_2 \leq \sqrt{np_n} \|\boldsymbol{\beta}_{\widetilde{\xi}^c}\|_2\leq \sqrt{np_n}\cdot \sqrt{p_n}\sigma a_n\leq c'\sqrt{n}\sigma\varepsilon_n$ for some constant $c'$,
and $$\|(\boldsymbol{X}^T_{\widetilde{\xi}}\boldsymbol{X}_{\widetilde{\xi}})^{-1}\boldsymbol{X}^T_{\widetilde{\xi}}\boldsymbol{X}_{\widetilde{\xi^c}}\boldsymbol{\beta}_{\widetilde{\xi}^c}\|_2/\sigma \leq\sqrt{\|(\boldsymbol{X}^T_{\widetilde{\xi}}\boldsymbol{X}_{\widetilde{\xi}})^{-1}\|_2}c'\sqrt{n} \varepsilon_n \leq \sqrt{1/n\lambda_0}\sqrt{n}c' \varepsilon_n \leq c'\varepsilon_n/\sqrt{\lambda_0},$$
where the second inequality is due to $|\widetilde{\xi}|\leq \widetilde{p}+s \leq \overline{p}$. Besides, 
\[
\|\boldsymbol{\beta}_{\widetilde{\xi}} - \boldsymbol{\beta}^0_{\widetilde{\xi}}\|_2 \geq \|\boldsymbol{\beta} - \boldsymbol{\beta}^0\|_2 - \sqrt p_n \sigma a_n.
\]
Therefore, $(\|\boldsymbol{\beta}_{\widetilde{\xi}} - \boldsymbol{\beta}^0_{\widetilde{\xi}}\|_2 - \sigma\varepsilon_n/2 - \|(\boldsymbol{X}^T_{\widetilde{\xi}}\boldsymbol{X}_{\widetilde{\xi}})^{-1}\boldsymbol{X}^T_{\widetilde{\xi}}\boldsymbol{X}_{\widetilde{\xi^c}}\boldsymbol{\beta}_{\widetilde{\xi}^c}\|_2)/\sigma\geq M_n\varepsilon_n/(2\sigma)$ when $M_n$ is sufficiently large, and
\[
\sup_{\boldsymbol{\beta} \in C_n}\mathbb{E}_{\boldsymbol{\beta}}(1-\phi_n) \leq \sup_{\boldsymbol{\beta} \in C_n}Pr\{\|(\boldsymbol{X}^T_{\widetilde{\xi}}\boldsymbol{X}_{\widetilde{\xi}})^{-1}\boldsymbol{X}^T_{\widetilde{\xi}}\boldsymbol{\epsilon}\|_2\geq M_n\varepsilon_n/(2\sigma)\} \leq \exp(-c_2n M_n^2\varepsilon^2_n).
\]
\end{proof}

As a technical tool, we restates the Donsker and Varadhan's representation for the KL divergence in the following lemma, whose proof can be found in \cite{Boucheron2013Concentration}.

\begin{lem} \label{lm:Donsker}
For any two probability measures $P$ and $Q$, and any measurable function $f$ such that $\int e^f dP < \infty$,
\[
\int f dQ \leq \mbox{KL}(Q\|P) + \log \int e^{f} dP.
\]
\end{lem}

The next two lemmas bound the contraction rate of $\widehat q(\boldsymbol{\beta})$ on $B_n$ and $B_n^c$ respectively.

\begin{lem} \label{lm:risk1}
With dominating probability, 
\[
\widehat q(B_n\cap\{\|\boldsymbol{\beta} - \boldsymbol{\beta}^0\|_2\geq M_n\varepsilon_n\}) = o(1),
\]
where $\varepsilon_n=\sqrt{\widetilde p\log(p_n\vee n)/n}$ and $M_n$ is any diverging sequence as $M_n \rightarrow \infty$.
\end{lem}
\begin{proof}
We denote  $\widetilde{\pi}(\boldsymbol{\beta})$ and $\widetilde{q}(\boldsymbol{\beta})$ as the truncated distribution of $\pi(\boldsymbol{\beta})$ and  $\widehat{q}(\boldsymbol{\beta})$ on set $B_n$, i.e.
\[
\begin{split}
\widetilde{\pi}(\boldsymbol{\beta})&=\pi(\boldsymbol{\beta})1(\boldsymbol{\beta}\in B_n)/\pi(B_n), \\
\widetilde{q}(\boldsymbol{\beta})&=\widehat{q}(\boldsymbol{\beta})1(\boldsymbol{\beta}\in B_n)/\widehat{q}(B_n).
\end{split}
\]
Define $V(P_{\boldsymbol{\beta}}, P_0) = M_n\varepsilon_n1(\|\boldsymbol{\beta} - \boldsymbol{\beta}^0\|_2\geq M_n\varepsilon_n)$ and
\[
\log \eta(P_{\boldsymbol{\beta}}, P_0)=l_n(P_{\boldsymbol{\beta}}, P_0) + \frac{n}{3}V^2(P_{\boldsymbol{\beta}}, P_0).
\] 
Lemma \ref{lm:testing} implies the existence of testing function within $B_n$ and by the same argument used in Theorem 3.1 of \cite{Pati2018on}, it can be shown that w.h.p.,
\[
\int_{B_n}\eta(P_{\boldsymbol{\beta}}, P_0)\widetilde{\pi}(\boldsymbol{\beta})d\boldsymbol{\beta} \leq e^{C'_1n\varepsilon^2_n}
\]
for some $C'_1>0$. By Lemma \ref{lm:Donsker}, it follows that w.h.p.,
\[
\begin{split}
&\frac{n}{3\widehat{q}(B_n)}M_n^2\varepsilon_n^2 \widehat q(B_n\cap\{\|\boldsymbol{\beta} - \boldsymbol{\beta}^0\|_2\geq M_n\varepsilon_n\})\\
=&\frac{n}{3\widehat{q}(B_n)}\int_{B_n}V^2(P_{\boldsymbol{\beta}}, P_0)\widehat{q}(\boldsymbol{\beta})d\boldsymbol{\beta}\\
=& \frac{n}{3}\int_{B_n}V^2(P_{\boldsymbol{\beta}}, P_0)\widetilde{q}(\boldsymbol{\beta})d\boldsymbol{\beta}\\
\leq & C'_1n\varepsilon^2_n + \mbox{KL}(\widetilde{q}(\boldsymbol{\beta})\|\widetilde{\pi}(\boldsymbol{\beta})) - \int_{B_n} l_n(P_{\boldsymbol{\beta}}, P_0)\widetilde{q}(\boldsymbol{\beta})d\boldsymbol{\beta}.
\end{split}
\]
Noting that,
\begin{equation*}
\begin{split}
&\mbox{KL}(\widetilde{q}(\boldsymbol{\beta})\|\widetilde{\pi}(\boldsymbol{\beta}))\\
=& \frac{1}{\widehat q(B_n)} \int_{B_n} \log \frac{\widehat{q}(\boldsymbol{\beta})}{\pi(\boldsymbol{\beta})} \widehat{q}(\boldsymbol{\beta})d\boldsymbol{\beta} + \log \frac{\pi(B_n)}{\widehat q(B_n)}\\
=&\frac{1}{\widehat q(B_n)} \mbox{KL}(\widehat{q}(\boldsymbol{\beta})
\|\pi(\boldsymbol{\beta})) - \frac{1}{\widehat q(B_n)} \int_{B_n^c} \log \frac{\widehat{q}(\boldsymbol{\beta})}{\pi(\boldsymbol{\beta})} \widehat{q}(\boldsymbol{\beta})d\boldsymbol{\beta}+ \log \frac{\pi(B_n)}{\widehat q(B_n)},
\end{split}
\end{equation*}
and similarly,
\begin{equation*}
\begin{split}
&\int_{B_n} l_n(P_{\boldsymbol{\beta}}, P_0)\widetilde{q}(\boldsymbol{\beta}) d\boldsymbol{\beta}  =\frac{1}{\widehat q(B_n)} \int l_n(P_{\boldsymbol{\beta}}, P_0)\widehat{q}(\boldsymbol{\beta}) d\boldsymbol{\beta} -\frac{1}{\widehat q(B_n)} \int_{B_n^c} l_n(P_{\boldsymbol{\beta}}, P_0)\widehat{q}(\boldsymbol{\beta}) d\boldsymbol{\beta}.
\end{split}
\end{equation*}

Combine the above three inequalities, we obtain that
\begin{align}
&M_n^2\varepsilon_n^2 \widehat q(B_n\cap\{\|\boldsymbol{\beta} - \boldsymbol{\beta}^0\|_2\geq M_n\varepsilon_n\}) \nonumber \\
\leq &C'\widehat{q}(B_n)\varepsilon^2_n + \frac{3}{n}\Bigl\{ \mbox{KL}(\widehat{q}(\boldsymbol{\beta})\|\pi(\boldsymbol{\beta})) - \int l_n(P_{\boldsymbol{\beta}}, P_0)\widehat{q}(\boldsymbol{\beta})d\boldsymbol{\beta}\Bigr\} \nonumber\\
&+ \frac{3}{n}\int_{B_n^c} l_n(P_{\boldsymbol{\beta}}, P_0)\widehat{q}(\boldsymbol{\beta})d\boldsymbol{\beta}
+\frac{3}{n}\int_{B_n^c} \log \frac{\pi(\boldsymbol{\beta})}{\widehat{q}(\boldsymbol{\beta})} \widehat{q}(\boldsymbol{\beta})d\boldsymbol{\beta} 
+ \frac{3\widehat{q}(B_n)}{n}\log \frac{\pi(B_n)}{\widehat q(B_n)}. \label{eq:sum}
\end{align}

By Theorem \ref{thm:elbo}, the second term in the RHS of (\ref{eq:sum}) is bounded by $3\varepsilon_n^2$.

Apply the similar argument used in the proof of Theorem \ref{thm:elbo}, the third term in the RHS of (\ref{eq:sum}) is bounded by 
% \[\begin{split}
%   & \frac{3}{n}\int_{B_n^c} l_n(P_{\boldsymbol{\beta}}, P_0)\widehat{q}(\boldsymbol{\beta}) d\boldsymbol{\beta}
%  = \frac{3}{2n\sigma^2}\int_{B_n^c} \left[\sum^n_{i=1}\epsilon_i^2 -\sum^n_{i=1} (\epsilon_i+\boldsymbol{X}_i\boldsymbol{\beta}^0-\boldsymbol{X}_i\boldsymbol{\beta})^2\right]\widehat{q}(\boldsymbol{\beta}) d\boldsymbol{\beta}\\
%  = &  \frac{3}{2n\sigma^2}\int_{B_n^c} \left[-2\sum^n_{i=1} (\epsilon_i\times(\boldsymbol{X}_i\boldsymbol{\beta}^0-\boldsymbol{X}_i\boldsymbol{\beta})-\sum^n_{i=1}(\boldsymbol{X}_i\boldsymbol{\beta}^0-\boldsymbol{X}_i\boldsymbol{\beta})^2\right]\widehat{q}(\boldsymbol{\beta}) d\boldsymbol{\beta}\\
%  = &\frac{3}{2n\sigma^2}\left\{
%  -2\sum^n_{i=1}\epsilon_i\int_{B_n^c}(\boldsymbol{X}_i\boldsymbol{\beta}^0-\boldsymbol{X}_i\boldsymbol{\beta})\widehat{q}(\boldsymbol{\beta})d\boldsymbol{\beta} - \int_{B_n^c}\sum^n_{i=1}(\boldsymbol{X}_i\boldsymbol{\beta}^0-\boldsymbol{X}_i\boldsymbol{\beta})^2\widehat{q}(\boldsymbol{\beta}) d\boldsymbol{\beta}
%  \right\}.
% \end{split}
% \]
\[
\begin{split}
   &\frac{3}{n}\int_{B_n^c} l_n(P_{\boldsymbol{\beta}}, P_0)\widehat{q}(\boldsymbol{\beta}) d\boldsymbol{\beta}\\
 = &\frac{3}{2n\sigma^2}\Bigl\{-2\sigma\boldsymbol{\epsilon}^T\int_{B_n^c}(\boldsymbol{X}\boldsymbol{\beta}^0 - \boldsymbol{X}\boldsymbol{\beta})\widehat{q}(\boldsymbol{\beta}) d\boldsymbol{\beta} -\int_{B_n^c} \|\boldsymbol{X}\boldsymbol{\beta}^0 - \boldsymbol{X}\boldsymbol{\beta}\|^2_2\widehat{q}(\boldsymbol{\beta}) d\boldsymbol{\beta}\Bigr\}.
\end{split}
\]
Note that $-2\sigma\boldsymbol{\epsilon}^T\int_{B_n^c}(\boldsymbol{X}\boldsymbol{\beta}^0 - \boldsymbol{X}\boldsymbol{\beta})\widehat{q}(\boldsymbol{\beta}) d\boldsymbol{\beta}$ follows a normal distribution $\mathcal{N}(0, V^2)$, where 
$V^2=4\sigma^2\|\int_{B_n^c}(\boldsymbol{X}\boldsymbol{\beta}^0-\boldsymbol{X}\boldsymbol{\beta})\widehat{q}(\boldsymbol{\beta})d\boldsymbol{\beta}\|^2\leq 4\sigma^2 \int_{B_n^c}\|\boldsymbol{X}\boldsymbol{\beta}^0-\boldsymbol{X}\boldsymbol{\beta}\|_2^2\widehat{q}(\boldsymbol{\beta})d\boldsymbol{\beta}$. Thus the third term in the RHS of (\ref{eq:sum}) is bounded by 
\begin{equation} \label{eq:third}
\frac{3}{2n\sigma^2}\left[\mathcal{N}(0,V^2)-\frac{V^2}{4\sigma^2}\right].
\end{equation}
Noting that $\mathcal{N}(0,V^2)=O_p(G_nV)$ for any diverging sequence $G_n$, (\ref{eq:third}) is further bounded, w.h.p., by
\[
\frac{3}{2n\sigma^2}(G_nV-\frac{V^2}{4\sigma^2})\leq \frac{3}{2n\sigma^2}\sigma^2G_n^2.
\] 
Therefore, the third term in the RHS of (\ref{eq:sum}) can be bounded by $\varepsilon_n^2$ w.h.p. (by choosing
$G_n^2\asymp n\varepsilon_n^2$).

The fourth term in the RHS of (\ref{eq:sum}) is bounded by 
\[
\begin{split}
\frac{3}{n} \int_{ B_n^c} \log \frac{\pi(\boldsymbol{\beta})}{\widehat{q}(\boldsymbol{\beta})} \widehat{q}(\boldsymbol{\beta})d\boldsymbol{\beta} \leq \frac{3}{n}\widehat q( B_n^c)\log \frac{\pi( B_n^c)}{\widehat q( B_n^c)}
\leq \frac{3}{n}\sup_{x\in(0,1)}[x\log(1/x)]=O(1/n).
\end{split}
\]
Similarly, the fifth term in the RHS of (\ref{eq:sum}) is bounded by $O(1/n)$.

Therefore, we have that w.h.p.,
\[
\begin{split}
    &M_n^2\varepsilon_n^2 \widehat q(B_n\cap\{\|\boldsymbol{\beta} - \boldsymbol{\beta}^0\|_2\geq M_n\varepsilon_n\})
    \leq C'\widehat{q}(B_n)\varepsilon^2_n + 3\varepsilon_n^2+\varepsilon_n^2 +1/n,
\end{split}
\]
that is, $\widehat q(B_n\cap\{\|\boldsymbol{\beta} - \boldsymbol{\beta}^0\|_2\geq M_n\varepsilon_n\})=O_p(1/M_n^2)=o_p(1)$.
\end{proof}

\begin{lem} \label{lm:risk2}
With dominating probability, $\widehat q(B_n^c) = o(1)$.
\end{lem}

\begin{proof}
By Theorem \ref{thm:elbo}, we have that w.h.p.,
\[\begin{split}
&\mbox{KL}(\widehat q(\boldsymbol{\beta})\|\pi(\boldsymbol{\beta}))+ \int l_n(P_0, P_{\boldsymbol{\beta}})\widehat q(\boldsymbol{\beta})d\boldsymbol{\beta} =\inf_{q(\boldsymbol{\beta}) \in \mathcal Q}\Bigl\{ \mbox{KL}(q(\boldsymbol{\beta})\|\pi(\boldsymbol{\beta}))
+ \int l_n(P_0, P_{\boldsymbol{\beta}})q(\boldsymbol{\beta})(d\boldsymbol{\beta}) \Bigr\}\\
\leq& Cn\varepsilon^2_n,
\end{split}
\] 
where $C$ is some constant.
By the similar argument used in the proof of Theorem \ref{thm:elbo} in the main text, 
\[
\int l_n(P_0, P_{\boldsymbol{\beta}})\widehat q(\boldsymbol{\beta})d\boldsymbol{\beta} \leq \frac{1}{2\sigma_\epsilon^2}\left(\int ||\boldsymbol{X}\boldsymbol{\beta} - \boldsymbol{X}\boldsymbol{\beta^0}||^2_2 \widehat q(\boldsymbol{\beta})(d\boldsymbol{\beta})
+Z
\right) 
\]
where $Z$ is a normal distributed $\mathcal{N}(0, \sigma^2c_0')$, where $c_0'\leq c_0=\int ||\boldsymbol{X}\boldsymbol{\beta} - \boldsymbol{X}\boldsymbol{\beta^0}||^2_2 \widehat q(\boldsymbol{\beta})(d\boldsymbol{\beta})$. Therefore, $-\int l_n(P_0, P_{\boldsymbol{\beta}})\widehat q(\boldsymbol{\beta})d\boldsymbol{\beta} = (1/2\sigma^2)[-c_0+O_p(\sqrt{c_0})]$, and 
$\mbox{KL}(\widehat q(\boldsymbol{\beta})\|\pi(\boldsymbol{\beta}))\leq Cn\varepsilon^2_n+(1/2\sigma^2)[-c_0+O_p(\sqrt{c_0})]=O_p(n\epsilon_n^2)$. 
%As $Cn\varepsilon^2_n\rightarrow\infty$, it follows that w.h.p., 
%\[
%\mbox{KL}(\widehat q(\boldsymbol{\beta})||\pi(\boldsymbol{\beta}))\leq Cn\varepsilon^2_n/2.
%\]

For any $\beta_j \sim \widehat{q}(\beta_j)$, define $\gamma_j = 1(|\beta_j/\sigma| > a_n)$, then
\begin{equation} \label{eq:kl2}
\begin{split}
&\mbox{KL}(\widehat{q}(\boldsymbol{\beta})\|\pi(\boldsymbol{\beta})) \geq \mbox{KL}(\widehat{q}(\boldsymbol{\gamma})\|\pi(\boldsymbol{\gamma})) \\
=&\sum^{p_n}_{j=1}\Bigl[\widehat{q}(\gamma_j=1)\log\frac{\widehat{q}(\gamma_j=1)}{\pi(\gamma_j=1)}+ \widehat{q}(\gamma_j=0)\log\frac{\widehat{q}(\gamma_j=0)}{\pi(\gamma_j=0)}\Bigr].
\end{split}
\end{equation}
Choose $\alpha_0 = p_n^{-1}$ and let $A = \{j:\widehat{q}(\gamma_j=1)\geq \alpha_0\}$, and denote $\alpha = \pi(\gamma_j=1)$. Noting that by the condition of $a_0$ and $b_n$, we can obtain that
\[\begin{split}
%&\alpha = \pi(\gamma_j = 1) \leq \frac{2\sqrt{\nu_0}s_0}{\sigma a_n (\nu_0-1)} \Bigl(1+\frac{(\sigma a_n)^2}{\nu_0s^2_0}\Bigr)^{-\frac{\nu_0-1}{2}} \leq \frac{2\sqrt{\nu_0}s_0}{\sigma a_n (\nu_0-1)}  \preceq \frac{s_0}{a_n} \asymp p^{\frac{2-m}{2}} \prec \alpha_0\\
&\alpha= \pi(\gamma_j = 1) \asymp \alpha_0/(n\vee p_n)^\delta,
\end{split}
\]
thus (\ref{eq:kl2}) implies $\sum_{j\in A}\widehat{q}(\gamma_j=1) \log(\alpha_0/\alpha) \leq C'n\varepsilon^2_n/2$ for some $C'$ and $\sum_{j\in A}\widehat{q}(\gamma_j=1) =O(\widetilde{p})$.

Under $\widehat{q}$, by Markov inequality,
\[
\begin{split}
&Pr(\sum_{j \in A}\gamma_j \geq \widetilde{p}/2) \leq 
Pr(\sum_{j \in A}\gamma_j \geq \widetilde{p}/3+\mathbb{E}\sum_{j \in A}\gamma_j)\leq 9\mbox{Var}(\sum_{j \in A}\gamma_j)/\widetilde p^2\leq 9\mathbb{E}\sum_{j \in A}\gamma_j/\widetilde p^2=o(1).\\
\end{split}
\]

%for $j \in A$, by Bernstein inequality of Binomial distribution tail,
%\[
%\begin{split}
%&Pr(\sum_{j \in A}\gamma_j \geq \widetilde{p}/2) \leq \exp \Bigl(-\frac{\widetilde{p}^2/8}{\sum_{j \in A}\mathbb{E}[\gamma^2_j] + \widetilde{p}/6} \Bigr) = \exp \Bigl(-\frac{\widetilde{p}^2/8}{\sum_{j \in A}\widehat{q}(\gamma_j=1) + \widetilde{p}/6} \Bigr)\\
%\leq &\exp(-c_3\widetilde{p}) \leq \exp(-c_3s)
%\end{split}
%\]

%If applying Chernoff bound, we need $\widetilde{p}/2 \geq \sum_{j\in A}\widehat{q}(\gamma_j=1)$, and 
%\[
%Pr(\sum_{j \in A}\gamma_j \geq \widetilde{p}/2) \leq exp(-\sum_{j\in A}\widehat{q}(\gamma_j=1))\Bigl(\frac{e\sum_{j\in A}\widehat{q}(\gamma_j=1)}{\widetilde{p}/2}\Bigr)^{\widetilde{p}/2}
%\]

On the other hand, under $\widehat{q}$, %for $j \notin A$,
%\[
%Pr(\sum_{j \notin A} \gamma_j \geq \widetilde{p}/2) \leq Pr(Bin(p, \alpha_0) \geq \widetilde{p}/2) \rightarrow Pr(Pois(1) %\geq \widetilde{p}/2) = O(\exp(-c_4s)).
%\]
by Chernoff bound,
\[
\begin{split}
&Pr(\sum_{j \notin A} \gamma_j \geq \widetilde{p}/2) \leq Pr(Bin(p_n, \alpha_0) \geq \widetilde{p}/2)\\
\leq& \exp\Bigl\{-p_n \Bigl(\frac{\widetilde{p}}{2p_n}\log \frac{\widetilde{p}/(2p_n)}{\alpha_0} + \Bigl(1-\frac{\widetilde{p}}{2p_n}\Bigr)\log\frac{1- \widetilde{p}/(2p_n)}{1-\alpha_0} \Bigr)\Bigr\} \leq \exp(-c\widetilde{p})=o(1),
\end{split}
\]
for some constant $c$, since $\widetilde p\rightarrow\infty$.

Combine the above results together, it is trivial to conclude that $\widehat q(B_n^c)=o(1)$.
\end{proof}

{\noindent \bf Proof of Theorem \ref{thm:contraction}}
\begin{proof}
Trivially combine Lemmas \ref{lm:risk1} and \ref{lm:risk2}, we obtain that $\widehat q(\{\|\boldsymbol{\beta} - \boldsymbol{\beta}^0\|_2\geq M_n\varepsilon_n\})=o_p(1)$ for any diverging $M_n$, where $\varepsilon_n=\sqrt{\widetilde p\log(p_n\vee n)/n}$. Due to the arbitrariness of $M_n$ and $\widetilde p$, we can let $M_n\sqrt{\widetilde p/s}\leq M_n'$, and the theorem naturally holds.
\end{proof}

\section{IMPLEMENTATION} \label{sec:implement}
The negative ELBO is
\begin{equation} \label{eq:elbo1}
\begin{split}
\Omega = &-\int \log p(\boldsymbol{Y}|\boldsymbol{\beta}, \lambda)q(\boldsymbol{\beta}|\boldsymbol{\lambda})q(\boldsymbol{\lambda})d\boldsymbol{\beta} d\boldsymbol{\lambda} + \int \mbox{KL}(q(\boldsymbol{\beta}|\boldsymbol{\lambda})\|\pi(\boldsymbol{\beta}|\boldsymbol{\lambda})) q(\boldsymbol{\lambda})d\boldsymbol{\lambda} + \mbox{KL}(q(\boldsymbol{\lambda})\|\pi(\boldsymbol{\lambda}))\\
= &const + \int \Bigl\{-\frac{\boldsymbol{Y}^T\boldsymbol{X}\mathbb{E}[\boldsymbol{\beta}|\boldsymbol{\lambda}]}{\sigma^2} + \frac{\mathbb{E}[\boldsymbol{\beta}^TX^TX\boldsymbol{\beta}|\boldsymbol{\lambda}]}{2\sigma^2}\Bigr\}q(\boldsymbol{\lambda}) d\boldsymbol{\lambda} 
+ \sum^{p_n}_{j=1} \int\Bigl[\log\frac{\lambda_j}{\lambda_j} 
+ \frac{\lambda_j^{-1}+\mu^2_j}{2\lambda_j^{-1}}\Bigr]q(\lambda_j)d\lambda_j\\
& + \sum^{p_n}_{j=1}\Bigl[a_0\log \frac{b_j}{b_n} - \log \frac{\Gamma(a_j)}{\Gamma(a_0)} + (a_j-a_0)\psi(a_j) - (b_j-b_n)\frac{a_j}{b_j}\Bigr]\\
= &const - \frac{\boldsymbol{Y}^T\boldsymbol{X}\boldsymbol{\mu}}{\sigma^2} + \frac{\boldsymbol{\mu}^T\boldsymbol{X}^T\boldsymbol{X}\boldsymbol{\mu}}{2\sigma^2} + \sum^{p_n}_{j=1}\int \Bigl\{\frac{n_j}{2\sigma^2}\lambda_j^{-1}\Bigr\}q(\lambda_j)d\lambda_j + \sum^{p_n}_{j=1} \int\Bigl( \frac{\mu^2_j\lambda_j}{2}\Bigr)q(\lambda_j)d\lambda_j \\
& + \sum^{p_n}_{j=1}\Bigl[a_0\log \frac{b_j}{b_n} - \log \frac{\Gamma(a_j)}{\Gamma(a_0)} + (a_j-a_0)\psi(a_j) - (b_j-b_n)\frac{a_j}{b_j}\Bigr],
\end{split}
\end{equation}
where $\psi(x)$ is the digamma function, and $n_j=[\boldsymbol{X}^T\boldsymbol{X}]_{j,j}$. Therefore,
\[
\begin{split}
\Omega = &const - \frac{\boldsymbol{Y}^T\boldsymbol{X}\boldsymbol{\mu}}{\sigma^2} + \frac{\boldsymbol{\mu}^T\boldsymbol{X}^T\boldsymbol{X}\boldsymbol{\mu}}{2\sigma^2} + \frac{1}{2\sigma^2}\sum^{p_n}_{j=1}\frac{n_jb_j}{a_j-1} + \sum^{p_n}_{j=1}(\mu^2_j/2+b_n)\frac{a_j}{b_j}\\
&+\sum^{p_n}_{j=1}\Bigl[a_0\log \frac{b_j}{b_n} - \log \frac{\Gamma(a_j)}{\Gamma(a_0)} + (a_j-a_0)\psi(a_j) - a_j\Bigr],
\end{split}
\]
and the gradients are
\[
\begin{split}
\frac{d\Omega}{d\boldsymbol{\mu}} &= -\frac{\boldsymbol{X}^T\boldsymbol{Y}}{\sigma^2}+\frac{\boldsymbol{X}^T\boldsymbol{X}\boldsymbol{\mu}}{\sigma^2} +  \boldsymbol{\Lambda} \mu,\\
\frac{d\Omega}{da_j} &= -\frac{n_j}{2\sigma^2}\frac{b_j}{(a_j-1)^2} + \frac{\mu^2_j/2+b_n}{b_j} + (a_j - a_0)\psi_1(a_j) - 1\\
\frac{d\Omega}{db_j} &= \frac{n_j}{2\sigma^2}\frac{1}{a_j-1} - \frac{(\mu^2_j/2+b_n)a_j}{b^2_j} + \frac{a_0}{b_j},
\end{split}
\]
where $\psi_1(x)$ is the trigamma function and $\boldsymbol{\Lambda}=\mbox{diag}(a_1/b_1,\dots,a_{p_n}/b_{p_n})$.
%\[
%\boldsymbol{\Lambda} = 
%\begin{bmatrix}
%\frac{a_1}{b_1} &  &\\
%&\cdots &\\
%& &\frac{a_p}{b_p}  
%\end{bmatrix}.
%\]
Solve the above equations, we have
\[
\begin{split}
    \boldsymbol{\mu} &= (\boldsymbol{X}^T\boldsymbol{X}+\sigma^2\boldsymbol{\Lambda})^{-1}\boldsymbol{X}^T\boldsymbol{Y}, \\
    a_j &= solve(-\frac{n_j}{2\sigma^2}\frac{b_j}{(a_j-1)^2} + \frac{\mu^2_j/2+b_n}{b_j} + (a_j - a_0)\psi_1(a_j) - 1=0),\\
    b_j &= \frac{-a_0+\sqrt{a^2_0+2n_ja_j(\mu^2_j/2 + b_n)/\sigma^2(a_j-1)}}{n_j/\sigma^2(a_j-1)}.
\end{split}
\]

If $\sigma$ is unknown, then the above derivation is modified as:
\[
\begin{split}
\Omega=&const + n\log\sigma + \frac{\boldsymbol{Y}^T\boldsymbol{Y}}{2\sigma^2}- \frac{\boldsymbol{Y}^T\boldsymbol{X}\boldsymbol{\mu}}{\sigma^2} + \frac{\boldsymbol{\mu}^T\boldsymbol{X}^T\boldsymbol{X}\boldsymbol{\mu}}{2\sigma^2} + \frac{1}{2\sigma^2}\sum^{p_n}_{j=1}\frac{n_jb_j}{a_j-1} + \sum^{p_n}_{j=1}(\mu^2_j/2+b_n)\frac{a_j}{b_j}\\
&+\sum^{p_n}_{j=1}\Bigl[a_0\log \frac{b_j}{b_n} - \log \frac{\Gamma(a_j)}{\Gamma(a_0)} + (a_j-a_0)\psi(a_j) - a_j\Bigr],
\end{split}
\]
and the additional partial derivative w.r.t. $\sigma$ is
\[
\begin{split}
\frac{d\Omega}{d\sigma}= \frac{n}{\sigma} - \frac{(\boldsymbol{Y}-\boldsymbol{X}\boldsymbol{\mu})^T(\boldsymbol{Y}-\boldsymbol{X}\boldsymbol{\mu})}{\sigma^3}-\frac{1}{\sigma^3}\sum^{p_n}_{j=1}\frac{n_jb_j}{a_j-1}.
\end{split}
\]
Thus, the updates of $\mu_j$, $a_j$ and $b_j$ keep the same, and the update of $\sigma$ follows
\[
\sigma = \sqrt{\frac{(\boldsymbol{Y}-\boldsymbol{X}\boldsymbol{\mu})^T(\boldsymbol{Y}-\boldsymbol{X}\boldsymbol{\mu})+\sum^{p_n}_{j=1}\frac{n_jb_j}{a_j-1}}{n}}.
\]

\section{COMPARISON OF MINIMIZING JOINT KL AND MARGINAL KL}

Our presented theory investigates the asymptotics of the variational Bayes distribution that minimizes the marginal KL divergence of $\boldsymbol{\beta}$. In such case, the negative ELBO is
\begin{equation} \label{eq:elbo2}
\widetilde \Omega = -\int \log p(\boldsymbol{Y}|\boldsymbol{\beta})q(\boldsymbol{\beta})d\boldsymbol{\beta} + \mbox{KL}(q(\boldsymbol{\beta})\|\pi(\boldsymbol{\beta})),
\end{equation}
where for $j=1, \ldots, p_n$,
\[
\pi(\beta_j) = \frac{1}{\sqrt{\nu_0}s_0}\Bigl(1 + \nu_0^{-1}\Bigl(\frac{\beta_j}{s_0}\Bigr)^2 \Bigr)^{-\frac{\nu_0+1}{2}},
\mbox{ and } q(\beta_j) = \frac{1}{\sqrt{\nu}s}\Bigl(1 + (\nu)^{-1}\Bigl(\frac{\beta_j -\widetilde{\mu}_j}{s_j}\Bigr)^2 \Bigr)^{-\frac{\nu+1}{2}},
\]
with $s_0 = \sqrt{b_n/a_0}$, $\nu_0 = 2a_0$, $s=\sqrt{b_j/a_j}$ and $\nu = 2a_j$. However $\mbox{KL}(q(\boldsymbol{\beta})\|\pi(\boldsymbol{\beta}))$ has no analytical expression, and the optimization of (\ref{eq:elbo2}) will then require Monte Carlo estimation and gradient descent type algorithms.

Therefore, for the simplicity of the computation, the ELBO optimization algorithm described in Section \ref{sec:implement} (as well as implemented in the simulation studies of main text) targets to minimize the joint KL divergence of $(\boldsymbol{\beta}, \boldsymbol{\lambda})$ rather than the marginal KL divergence of $\boldsymbol{\beta}$. In other words, there is a gap between our computational algorithm and our theory.

To justify that our implemented procedure (i.e., minimizing the joint KL divergence of $(\boldsymbol{\beta}, \boldsymbol{\lambda})$) is a close approximation of the variational procedure studied by our theory (i.e., minimizing the marginal KL divergence of $(\boldsymbol{\beta})$), We compare the two procedures via a toy example. Specifically, we would like to compare variational posterior means $\boldsymbol{\mu}$ (by minimizing (\ref{eq:elbo1})) and $\widetilde{\boldsymbol{\mu}}$ (by minimizing (\ref{eq:elbo2})). Consider a linear model with $n=100$, $p_n=100$ and $\boldsymbol{\beta}^0=(10,10,10,10,10,0,\ldots, 0)^T$. Suppose $\sigma$ is known and equals 1. For both two procedures, we choose $a_0=2$ and $b_n/a_0=\log(p_n)/[np_n^{2+1/a_0}p_n^{6/a_0}]$. We use Lasso estimator for both the initial value of $\boldsymbol{\mu}$ and $\widetilde{\boldsymbol{\mu}}$, and Adam \citep{kingma2015adam} is used for minimizing (\ref{eq:elbo2}) with the learning rate being 0.001.

We run the experiment for 100 times, and the means and standard deviations of the mean squared error (MSE) $(\sum^K_{k=1}(\mu_k - \widetilde{\mu}_k)^2/K)$ for both the nonzero entries $\boldsymbol{\beta}_{\xi^0}$ and zero entries $\boldsymbol{\beta}_{(\xi^0)^c}$ are reported: For $\boldsymbol{\beta}_{\xi^0}$, the MSE is $0.0108 \pm 0.0038$; For $\boldsymbol{\beta}_{(\xi^0)^c}$, the MSE is $0.0006 \pm 0.0008$.

This toy example shows there is little estimation difference in minimizing (\ref{eq:elbo1}) or (\ref{eq:elbo2}), and thus in practice the Algorithm 1 in the main text is preferred due to its simple form of coordinate descent update.

\end{document}